\documentclass[10pt,letterpaper]{article}

\usepackage{times}
\usepackage{epsfig}
\usepackage{graphicx}
\usepackage{amsmath}
\usepackage{amssymb}
\usepackage{algpseudocode}
\usepackage{algorithm}
\usepackage{lmodern}
\usepackage{subcaption}
\usepackage{color}
\usepackage{comment}

\usepackage{amsthm}
\newtheorem{theorem}{Theorem}

\newcommand\numberthis{\addtocounter{equation}{1}\tag{\theequation}}

\usepackage{hyperref}

\DeclareMathOperator*{\argmax}{arg\,max}

\begin{document}

\title{An efficient Exact-PGA algorithm for constant curvature manifolds}

\author{
Rudrasis Chakraborty\textsuperscript{1}, Dohyung Seo\textsuperscript{2} and Baba C. Vemuri\textsuperscript{1} \\
\textsuperscript{1}Department of CISE, University of Florida\\
\textsuperscript{2}U-Systems, A GE Healthcare Company\\
\textsuperscript{1}{\tt\small \{rudrasis, vemuri\}@cise.ufl.edu} ~~~~\textsuperscript{2}{\tt\small dhseo.118@gmail.com}
}

\date{}

\maketitle

\begin{abstract}
Manifold-valued datasets are widely encountered in
many computer vision tasks. A non-linear analog of
the PCA, called the Principal Geodesic Analysis (PGA)
suited for data lying on Riemannian manifolds was reported
in literature a decade ago. Since the objective
function in PGA is highly non-linear and hard to solve
efficiently in general, researchers have proposed a linear
approximation. Though this linear approximation
is easy to compute, it lacks accuracy especially when
the data exhibits a large variance. Recently, an alternative
called exact PGA was proposed which tries to
solve the optimization without any linearization. For
general Riemannian manifolds, though it gives better
accuracy than the original (linearized) PGA, for data
that exhibit large variance, the optimization is not computationally
efficient. In this paper, we propose an efficient
exact PGA for constant curvature Riemannian
manifolds (CCM-EPGA). CCM-EPGA differs significantly
from existing PGA algorithms in two aspects,
(i) the distance between a given manifold-valued data
point and the principal submanifold is computed analytically
and thus no optimization is required as in
existing methods. (ii) Unlike the existing PGA algorithms,
the descent into codimension-1 submanifolds
does not require any optimization but is accomplished
through the use of the Rimeannian inverse Exponential
map and the parallel transport operations. We present
theoretical and experimental results for constant curvature
Riemannian manifolds depicting favorable performance
of CCM-EPGA compared to existing PGA
algorithms. We also present data reconstruction from
principal components and directions which hasn’t been
presented in literature in this setting.

\end{abstract}

\section{Introduction}\label{sec1}

Principal Component Analysis (PCA) is a widely used dimensionality
reduction technique in Science and Engineering. PCA however requires
the input data to lie in a vector space. With the advent of new
technologies and wide spread use of sophisticated feature extraction
methods, manifold-valued data have become ubiquitous in many fields
including Computer Vision, Medical Imaging and Machine Learning. A
nonlinear version of PCA, called Principal Geodesic Analysis (PGA),
for data lying on Riemannian manifolds was introduced in
\cite{fletcher2004principal}.

Since the objective function of PGA is highly nonlinear and hard to
solve in general, researchers proposed a linearized version of PGA
\cite{fletcher2004principal}. Though this linearized PGA, hereafter
referred to as PGA, is computationally efficient, it lacks accuracy
for data with large spread/variance.  In order to solve the objective
function exactly, some researchers proposed to solve the original
objective function (not the approximation) and called it {\it exact
  PGA} \cite{sommer2010manifold}.  While {\it exact PGA} attempts to
solve this complex nonlinear optimization problem, it is however
computationally inefficient. Though it is not possible to efficiently
and accurately solve this optimization problem for a general manifold,
however, for manifolds with constant sectional curvature, we formulate
an efficient and exact PGA algorithm, dubbed CCM-EPGA. It is well
known in geometry, by virtue of Killing-Hopf theorem,
\cite{boothby1986introduction} that any non-zero constant curvature
manifold is isomorphic to either the hypersphere ($\mathbf{S}^N$) or
the hyperbolic space ($\mathbf{H}^N$), hence in this work, we present
the CCM-EPGA formulation for ($\mathbf{S}^N$) and
($\mathbf{H}^N$). Our formulation has several applications to Computer
Vision and Statistics including directional data
\cite{mardia1989shape} and color spaces
\cite{lenz2007hyperbolic}. Several other applications of hyperbolic
geometry are, shape analysis \cite{zeng2010ricci}, electrical
Impedence Tomography, Geoscience imaging \cite{uhlmann2013inverse},
brain morphometry \cite{wang2009teichmuller}, catadiaoptric vision
\cite{bogdanova2007scale} etc.

In order to depict the effectiveness of our proposed CCM-EPGA, we use
the {\it average projection error} as defined in
\cite{sommer2010manifold}. We also report the computational time
comparison of the CCM-EPGA with the PGA \cite{fletcher2004principal}
and the {\it exact PGA} \cite{sommer2010manifold}. Several variants of
PGA exist in literature and we briefly mention a few here.  In
\cite{said2007exact}, authors computed the principal geodesics
without approximation only for a special Lie group, $SO(3)$. Geodesic
PCA (GPCA) \cite{huckemann2006principal,huckemann2010} solves a
different optimization function namely, optimizing the projection
error along geodesics. The authors in \cite{huckemann2010} minimize
the projection error instead of maximizing variance in geodesic
subspaces (defined later in the paper). GPCA does not use a linear
approximation, but it is restricted to manifolds where a closed form
expression for the geodesics exists.  More recently, a probabilistic version
of PGA called PPGA was presented in \cite{zhang2013probabilistic},
which is a nonlinear version of PPCA
\cite{tipping1999probabilistic}. None of these methods attempt to
compute the solution to the exact PGA problem defined in
\cite{sommer2010manifold}.

The rest of the paper is organized as follows. In Section \ref{sec2},
we present the formulation of PGA. We also discuss the details of the
linearized version of PGA \cite{fletcher2004principal} and {\it exact
  PGA} \cite{sommer2010manifold}. Our formulation of CCM-EPGA is
presented in Section \ref{sec2}. We present experimental results of
CCM-EPGA algorithm along with comparisons to {\it exact PGA} and PGA
in Section \ref{sec3}. In addition to synthetic data experiments, we
present the comparative performance of CCM-EPGA on two real data
applications. In Section \ref{sec4}, we present the formulation for
reconstruction of data from principal directions and components in
this nonlinear setting. Finally, in section \ref{sec5}, we draw conclusions.

\section{Principal Geodesic Analysis}
\label{sec2}
Principal Component Analysis (PCA) \cite{jolliffe2002principal} is a
well known and widely used statistical method for dimensionality
reduction. Given a vector valued dataset, it returns a sequence of
linear subspaces that maximize the variance of the projected data. The
$k^{th}$ subspace is spanned by the principal vectors $\{\mathbf{v}_1,
\mathbf{v}_2, \cdots, \mathbf{v}_k\}$ which are mutually
orthogonal. PCA is well suited for vector-valued data sets but not for
manifold-valued inputs. A decade ago, the nonlinear version called the
Principal Geodesic Analysis (PGA) was developed to cope with
manifold-valued inputs \cite{fletcher2004principal}. In this section,
first, we briefly describe this PGA algorithm, then, we show the key
modification performed in \cite{sommer2010manifold} to arrive at what
they termed as the exact PGA algorithm. We then motivate and present
our approach which leads to an efficient and novel algorithm for exact
PGA on constant curvature manifolds (CCM-EPGA).

Let $M$ be a Riemannian manifold. Let's suppose we are given a
dataset, $X = \{x_1, \cdots, x_n\}$, where $x_j \in M$. Let us assume
that the finite sample Fr\'{e}chet mean \cite{frechet1948elements} of
the data set exists and be denoted by $\mu$. Let $V_k$ be the space
spanned by mutually orthogonal vectors (principal directions)
$\{\mathbf{v}_1, \cdots, \mathbf{v}_k\}$, $\mathbf{v}_j \in T_{\mu}M,
\forall j$. Let $S_k$ be the $k^{th}$ geodesic subspace of $T_{\mu}M$,
i.e., $S_k = Exp_{\mu}(V_k)$, where $Exp$ is the Riemannian
exponential map (see \cite{boothby1986introduction} for definition). Then, the
principal directions, $\mathbf{v}_i$ are defined recursively by
\begin{eqnarray}
\label{sec2:eq1}
\mathbf{v}_i &=& \displaystyle\argmax_{\|\mathbf{v}\|=1, \mathbf{v} \in V_{i-1}^{\bot}}\frac{1}{n}\displaystyle\sum_{j=1}^n d^2(\mu, \Pi_{S_i}(x_j)) \\
S_i &=& Exp_{\mu}(span{V_{i-1}, \mathbf{v}_i})
\end{eqnarray}
where $d(x,y)$ is the geodesic distance between $x \in M$ and $y \in
M$, $\Pi_{S}(x)$ is the point in $S$ closest to $x \in M$.

The PGA algorithm on $M$ is summarized in Alg. \ref{sec2:alg1}.

\begin{algorithm}
  \caption{The PGA algorithm on manifold $M$}
\label{sec2:alg1}

  \begin{algorithmic}[1]
    \State Given dataset $X = \{x_1, \cdots, x_n\} \in M$, and $1\leq L \leq dim(M)$
   \State Compute the FM, $\mu$, of $X$ \cite{afsari2011riemannian}
    \State Set $k \leftarrow 1$
    \State Set  $\{\bar{x}_1^0, \cdots, \bar{x}_n^0\}$ $\leftarrow$ $\{x_1, \cdots, x_n\}$ 
    \While{  $k \leq L$  } 
    \State Solve $\mathbf{v}_k = \displaystyle\argmax_{\|\mathbf{v}\|=1, \mathbf{v} \in T_{\mu}M, \mathbf{v} \in V_{k-1}^{\bot}}\frac{1}{n}\displaystyle\sum_{j=1}^n d^2(\mu, \Pi_{S_k}(x_j))$ as in Eq. \eqref{sec2:eq1}.
     \State Project $\{\bar{x}_1^{k-1}, \cdots, \bar{x}_n^{k-1}\}$ to a $k$ co-dimension one submanifold $Z$ of $M$, which is orthogonal to the current geodesic subspace.
    \State Set the projected points to $\{\bar{x}_1^k, \cdots, \bar{x}_n^k\}$ 
    \State $k \leftarrow k+1$ 
    \EndWhile  
  \end{algorithmic}
\end{algorithm}
\subsection{PGA and exact PGA}

In Alg. \ref{sec2:alg1} (lines $6-7$), as the projection operator
$\Pi$ is hard to compute, hence a common alternative is to locally
linearize the manifold. This approach \cite{fletcher2004principal}
maps all data points on the tangent space at $\mu$, and as tangent
plane is a vector space, one can use PCA to get the principal
directions. This simple scheme is an approximation of the PGA. This
motivated the researchers to ask the following question: \emph{Is it
  possible to do PGA, i.e., solve Eq. \eqref{sec2:eq1} without any
  linearization?} The answer is yes. But, computation of the
projection operator, $\Pi_{S}(x)$, i.e., the closest point of $x$ in
$S$ is computationally expensive.  In \cite{sommer2010manifold},
Sommer et al. give an alternative definition of PGA, i.e., they
minimize the average squared reconstruction error, i.e., $d^2(x_j,
\Pi_{S_i}(x_j))$ instead of $d^2(\mu, \Pi_{S_i}(x_j))$ in
eqns. \eqref{sec2:eq1}. They use an optimization scheme to compute
this projection. Further, they termed their algorithm, \emph{exact
  PGA}, as it does not require any linearization. {\it However, their
  optimization scheme is in general computationally expensive and for
  a data set with large variance, convergence is not
  guaranteed}. Hence, for large variance data, their \emph{exact PGA}
is still an approximation as it might not converge. This motivates us
to formulate an accurate and computationally efficient \emph{exact
  PGA}.
\subsection{Efficient and accurate exact PGA}

In this paper, we present an analytic expression for the projected
point and design an effective way to project data points on to the
co-dimension $k$ submanifold (as in \ref{sec2:alg1}, line $7$). An
analytic expression is in general not possible to derive for arbitrary
Riemannian manifolds. However, for constant curvature Riemannian
manifolds, i.e., $\mathbf{S}^N$ (positive constant curvature) and
$\mathbf{H}^N$ (negative constant curvature), we derive an analytic
expression for the projected point and devise an efficient algorithm
to project data points on to a co-dimension $k$ submanifold. Both
these manifolds are quite commonly encountered in Computer Vision
\cite{hartley2013rotation,lenz1998if,bogdanova2007scale,wang2009teichmuller,zeng2010ricci}
as well as many other fields of Science and Engineering. The former
more so than the latter.  Eventhough, there are applications that can
be naturally posed in hyperbolic spaces (e.g., color spaces in Vision
\cite{lenz2007hyperbolic}, catadiaoptric images
\cite{bogdanova2007scale} etc.), their full potential has not yet been
exploited in Computer Vision research as much as in the former case.

We first present some background material for the $N$-dimensional
spherical and hyperbolic manifolds and then derive an analytical
expression of the projected point.
\subsubsection{Basic Riemannian Geometry of $\mathbf{S}^N$} 
\begin{itemize}
\item \textbf{Geodesic distance}: The geodesic distance between $\psi,
  \bar{\psi} \in \mathbf{S}^N$ is given by, $d(\psi, \bar{\psi}) =
  \arccos(\psi^t\bar{\psi})$.
\item \textbf{Exponential Map}: Given a vector $\mathbf{v} \in
  T_{\psi}\mathbf{S}^N$, the Riemannian Exponential map on
  $\mathbf{S}^N$ is defined as $Exp_{\psi}(\mathbf{v}) =
  \cos(|\mathbf{v}|)\psi +
  \sin(|\mathbf{v}|)\mathbf{v}/|\mathbf{v}|$. The Exponential map
  gives the point which is located on the great circle along the
  direction defined by the tangent vector $\mathbf{v}$ at a distance
  $|\mathbf{v}|$ from $\psi$.
\item \textbf{Inverse Exponential Map}: The tangent vector $\mathbf{v}
  \in T_{\psi} \mathbf{S}^N$ directed from $\psi$ to $\bar{\psi}$ is
  given by, $Exp^{-1}_{\psi}(\bar{\psi}) = \frac{\theta}{\sin(\theta)}
  (\bar{\psi} - \psi\cos(\theta))$ where, $\theta = d(\psi,
  \bar{\psi})$. Note that, the inverse exponential map is defined
  everywhere on the hypersphere except at the antipodal points.
\end{itemize}
\subsubsection{Basic Riemannian Geometry of $\mathbf{H}^N$} 

The hyperbolic $N$-dimensional manifold can be embedded in
$\mathbf{R}^{N+1}$ using three different models. In this discussion,
we have used the hyperboloid model \cite{iversen1992hyperbolic}. In this model,
$\mathbf{H}^N$ is defined as $\mathbf{H}^N = \{\mathbf{x} = (x_1,
\cdots, x_{N+1})^t \in \mathbf{R}^{N+1} | <x,x>_H = -1, x_1 > 0\}$,
where the inner product on $\mathbf{H}^N$, denoted by $<x,y>_H$ is defined as $<x,y>_H = -x_1*y_1 + \sum_{i=2}^{N+1} (x_i*y_i)$.
\begin{itemize}
\item \textbf{Geodesic distance}: The geodesic distance between $\psi,
  \bar{\psi} \in \mathbf{H}^N$ is given by, $d(\psi, \bar{\psi}) =
  \cosh^{-1}(-<\psi,\bar{\psi}>_H)$.
\item \textbf{Exponential Map}: Given a vector $\mathbf{v} \in
  T_{\psi}\mathbf{H}^N$, the Riemannian Exponential map on
  $\mathbf{H}^N$ is defined as, $Exp_{\psi}(\mathbf{v}) =
  \cosh(|\mathbf{v}|)\psi +
  \sinh(|\mathbf{v}|)\mathbf{v}/|\mathbf{v}|$.
\item \textbf{Inverse Exponential Map}: The tangent vector $\mathbf{v}
  \in T_{\psi} \mathbf{H}^N$ directed from $\psi$ to $\bar{\psi}$ is
  given by $Exp^{-1}_{\psi}(\bar{\psi}) = \frac{\theta}{\sinh(\theta)}
  (\bar{\psi} - \psi\cosh(\theta))$ where, $\theta = d(\psi,
  \bar{\psi})$.
\end{itemize}

For the rest of this paper, we consider the underlying manifold, $M$,
as a constant curvature Riemannian manifold, i.e., $M$ is
diffeomorphic to either $\mathbf{S}^N$ or $\mathbf{H}^N$, where $N =
dim(M)$ \cite{boothby1986introduction}. Let $\psi, \bar{\psi} \in M$, $\mathbf{v} \in
T_{\bar{\psi}}M$. Further, let $y(\mathbf{v}, \psi)$ be defined as the
projection of $\psi$ on the geodesic submanifold defined by
$\bar{\psi}$ and $\mathbf{v}$. Now, we will derive a closed form
expression for $y(\mathbf{v}, \psi)$ in the case of $\mathbf{S}^N$ and
$\mathbf{H}^N$.

\begin{figure}[!ht]
\centering
   \includegraphics[width=0.8\linewidth]{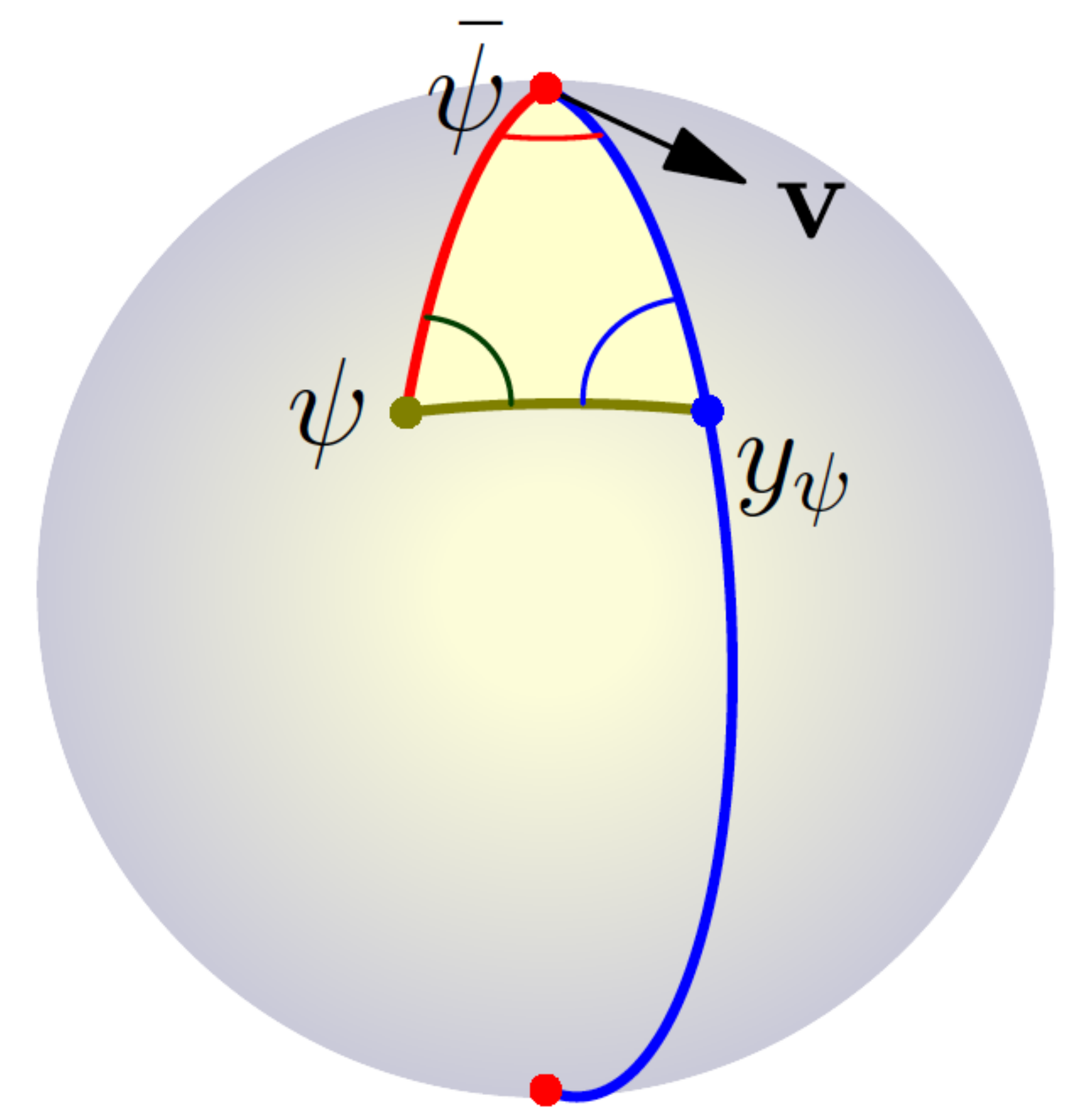}
\caption{Analytic expression for the projected point on the Sphere.}
\label{sec2:fig1}
\end{figure}

\subsection{Analytic expression for $y(\mathbf{v}, \psi)$ on $\mathbf{S}^N$}

\begin{theorem}
Let $\psi \in \mathbf{S}^N$ and $\mathbf{v} \in
T_{\bar{\psi}}\mathbf{S}^N$. Then the projection of $\psi$ on the
geodesic submanifold defined by $\bar{\psi}$ and $\mathbf{v}$, i.e.,
$y(\mathbf{v}, \psi)$ is given by:

\begin{align*}
y(\mathbf{v}, \psi) = \cos(\arctan \Bigg(\frac{(<\mathbf{v}, \psi>)/(<\psi, \bar{\psi}>)}{|\mathbf{v}|}\Bigg))\bar{\psi} \\ + \sin(\arctan \Bigg(\frac{(<\mathbf{v}, \psi>)/(<\psi, \bar{\psi}>)}{|\mathbf{v}|}\Bigg))\mathbf{v}/|\mathbf{v}| \numberthis 
\end{align*}

\end{theorem}

\begin{proof}
Consider the spherical triangle shown in Fig. \ref{sec2:fig1},
$\bigtriangleup \bar{\psi}\psi y_{\psi}$, where $y_{\psi} =
y(\mathbf{v}, \psi)$. Let, $a = d(\bar{\psi}, y_{\psi})$, $b =
d(y_{\psi}, \psi)$ and $c = d(\psi, \bar{\psi})$. Also, let $A =
\angle \bar{\psi}\psi y_{\psi}$, $B = \angle \psi\bar{\psi} y_{\psi}$,
$C = \angle \bar{\psi} y_{\psi} \psi$. Clearly, since $y_{\psi}$ is the
projected point , $C=\pi/2$.  So,
\begin{align*}
\cos B = \frac{<\frac{c}{\sin c}(\psi - \bar{\psi} \cos c), \mathbf{v}>}{c|\mathbf{v}|} \\
= \frac{\frac{<\mathbf{v}, \psi>}{\sin c} - \cot c <\mathbf{v}, \bar{\psi}>}{|\mathbf{v}|} \numberthis
\end{align*}
Here, $<.,.>$ denotes the Euclidean inner product, where both $\psi$
and $\mathbf{v}$ are viewed as points in $\mathbf{R}^{N+1}$, i.e., the
ambient space. Note that, $<\mathbf{v}, \bar{\psi}> = 0$, as
$\mathbf{v} \in T_{\bar{\psi}}\mathbf{S}^N$. From spherical
trigonometry, we know that $\tan a = \cos B \tan c$.
\begin{eqnarray*}
\therefore \cos B \tan c &=&  \frac{\frac{<\mathbf{v}, \psi>}{\cos c}}{|\mathbf{v}|} \\
&=& \frac{(<\mathbf{v}, \psi>)/(<\psi, \bar{\psi}>)}{|\mathbf{v}|}
\end{eqnarray*}
\begin{align}
\therefore a =\arctan \Bigg(\frac{(<\mathbf{v}, \psi>)/(<\psi, \bar{\psi}>)}{|\mathbf{v}|}\Bigg)
\end{align}
Hence, using the Exponential map, we can show that $y_{\psi}$ is given by,
\begin{equation}
\label{sec2:eq2}
y_{\psi} = \cos(a)\bar{\psi} + \sin(a)\mathbf{v}/|\mathbf{v}|
\end{equation}
\end{proof}

Analogously, we can derive the formula for $y(\mathbf{v}, \psi)$ on
$\mathbf{H}^N$, $\mathbf{v} \in T_{\bar{\psi}}\mathbf{H}^N$.

\begin{theorem}
Let $\psi \in \mathbf{H}^N$ and $\mathbf{v} \in
T_{\bar{\psi}}\mathbf{H}^N$. Then the projection of $\psi$ on the
geodesic submanifold defined by $\bar{\psi}$ and $\mathbf{v}$, i.e.,
$y(\mathbf{v}, \psi)$ is given by:
\begin{equation}
\label{sec2:eq3}
y_{\psi} = \cosh(a)\bar{\psi} + \sinh(a)\mathbf{v}/|\mathbf{v}|
\end{equation}
where, $$ a =\tanh^{-1} \Bigg(\frac{(<\mathbf{v}, \psi>_H)/(-<\psi,
  \bar{\psi}>_H)}{|\mathbf{v}|}\Bigg) $$.

\end{theorem}

\begin{proof}
As before, consider the hyperbolic triangle shown in, $\bigtriangleup
\bar{\psi}\psi y_{\psi}$, where $y_{\psi} = y(\mathbf{v}, \psi)$. Let,
$a = d(\bar{\psi}, y_{\psi})$, $b = d(y_{\psi}, \psi)$ and $c =
d(\psi, \bar{\psi})$. Also, let $A = \angle \bar{\psi}\psi y_{\psi}$,
$B = \angle \psi\bar{\psi} y_{\psi}$, $C = \angle \bar{\psi} y_{\psi}
\psi$. Clearly, since $y_{\psi}$ is the projected point , $C=\pi/2$.
Then, $B$ is the angle between $Log_{\bar{\psi}}(\psi)$ and
$\mathbf{v}$. Hence,
\begin{eqnarray}
\cosh B &=& \frac{\frac{<\mathbf{v}, \psi>_H}{\sinh c} - \coth c <\mathbf{v}, \bar{\psi}>_H}{|\mathbf{v}|}
\end{eqnarray}
Then, from hyperbolic trigonometry, as $\tanh a = \cosh B \tanh c$, we get 
\begin{align}
a =\tanh^{-1} \Bigg(\frac{(<\mathbf{v}, \psi>_H)/(-<\psi, \bar{\psi}>_H)}{|\mathbf{v}|}\Bigg) 
\end{align}

Note that, as the arc length between $\bar{\psi}$ and $y_{\psi}$ is
$a$, hence, using Exponential map, we can show that $y_{\psi} =
\cosh(a)\bar{\psi} + \sinh(a)\mathbf{v}/|\mathbf{v}|$.

\end{proof}

Given the closed form expression of the projected point, now we are in
a position to devise an efficient projection algorithm (for line $7$
in Alg. \ref{sec2:alg1}), which is given in Alg. \ref{sec2:alg2}. Note
that, using Alg. \ref{sec2:alg2}, data points on the current submanifold
are projected to a submanifold of dimension one less, which is needed
by the PGA algorithm in line 6. {\it Also note that, in order to
  ensure existence and uniqueness of FM, we have restricted the data
  points to be within a geodesic ball of convexity radius $< \pi/2$
  \cite{afsari2011riemannian}}.
\begin{algorithm}
  \caption{Algorithm for projecting the data points to a co-dimension one
    submanifold}\label{sec2:alg2}
  \begin{algorithmic}[1]
    \State \emph{Input}: a data point $x_i \in \mathbf{S}^N
    (\mathbf{H}^N)$, a geodesic submanifold defined at $\mu$ and
    $\mathbf{v} \in T_{\mu} \mathbf{S}^N (T_{\mu} \mathbf{H}^N)$, and
    $y(\mathbf{v}, x_i)$ which is the projection of $\psi$ on to the
    geodesic submanifold.  

\State \emph{Output}: $\bar{x}_i$ which is the projection of the data
point $x_i$ to a subspace, $\mathbf{S}^{N-1} (\mathbf{H}^{N-1})$,
which is orthogonal to the current geodesic submanifold.

\State Step 1. Evaluate the tangent vector, $\mathbf{v}_i \in
T_{y(\mathbf{v}, x_i)}\mathbf{S}^N (T_{y(\mathbf{v},
  x_i)}\mathbf{H}^N)$ directed towards $x_i$ using the Inverse
Exponential Map.  It is clear that $\mathbf{v}_i$ is orthogonal to
$\mathbf{v}$.

\State Step 2. Parallel transport $\mathbf{v}_i$ to
$\mu$. Let $\mathbf{v}_i^{\mu}$ denote the parallel transported
vector. The geodesic submanifold defined by $\mu$ and
$\mathbf{v}_i^{\mu}$ is orthogonal to geodesic submanifolds obtained from
the previous steps in Alg. \ref{sec2:alg1}.  

\State Step 3. Set $\bar{x}_i \leftarrow y(\mathbf{v}_i^{\mu}, x_i)$
  \end{algorithmic}
\end{algorithm}

Note that in order to descend to the codimension-1 submanifolds, we
use step-1 and step-2 instead of the optimization method used in the
exact PGA algorithm of \cite{sommer2010manifold}.

\section{Experimental Results}
\label{sec3}
In this section, we present experiments demonstrating the performance
of CCM-EPGA compared to PGA \cite{fletcher2004principal} and {\it
  exact PGA} \cite{sommer2010manifold}. We used the {\it average
  projection error}, defined in \cite{sommer2010manifold}, as a
measure of performance in our experiments. The {\it average projection
  error} is defined as follows. Let $\{x_i\}_{i=1}^n$ be data points
on a manifold $M$. Let $\mu$ be the mean of the data points. Let,
$\mathbf{v}$ be the first principal direction and
$S=Exp_{\mu}(\mathbf{v})$. Then the error $(E)$ is defined as follows:
\begin{equation}
E = \frac{1}{n}\sum_{i=1}^n d^2(x_i, \Pi_{S}(x_i))
\end{equation}
where $d(.,.)$ is the geodesic distance function on $M$.  We also
present the computation time for each of the three algorithms. All
the experimental results reported here were obtained on a desktop with
a single 3.33 GHz Intel-i7 CPU with 24 GB RAM.

\subsection{Comparative performance of CCM-EPGA on Synthetic data}
\label{exp1}

In this section, we present the comparitive performance of CCM-EPGA on
several synthetic datasets. For each of the synthetic data, we have
reported the {\it average projection error} and computation time for
all three PGA algorithms in Table \ref{tab1}.  All the four datasets
are on $S^2$ and the Fr\'{e}chet mean is at the "north pole". For all
the datasets, samples are in the northern hemisphere to ensure that
the Fr\'{e}chet mean is unique. Data\_1 and Data\_2 are generated by
taking samples along a geodesic with a slight perturbation. The last
two datasets are constructed by drawing random samples on the northern
hemisphere. In addition, data points from Data\_1 are depicted in
Fig. \ref{fig1}. The first principal direction is also shown (black
for CCM-EPGA, blue for PGA and red for {\it exact PGA}). Further, we
also report the data variance for these synthetic datasets. By
examining the results, it's evident that for data with low variance,
the significance of CCM-EPGA in terms of projection error is marginal,
while for high variance data, CCM-EPGA yields significantly better
accuracy. Also, CCM-EPGA is computationally very fast in comparison to
{\it exact PGA}. The results in Table \ref{tab1} indicate that
CCM-EPGA outperforms {\it exact PGA} in terms of efficiency and
accuracy. Although, the required time for PGA is less than that of
CCM-EPGA, in terms of accuracy, CCM-EPGA dominates PGA.

\begin{table*}
\begin{center}
\begin{tabular}{|l|c|c|c|c|c|c|c|}
\hline
Data & Var. & \multicolumn{2}{c|}{CCM-EPGA} & \multicolumn{2}{c|}{PGA} & \multicolumn{2}{c|}{\it exact PGA} \\ \hline
& &{\it avg. proj. err.} & Time(s) & {\it avg. proj. err.} & Time(s) & {\it avg. proj. err.} & Time(s)  \\
\hline\hline
Data\_1 & 2.16 & $\mathbf{1.13e-04}$ & 0.70 &0.174  & $\mathbf{0.46}$ & 2.54e-02 & 14853\\
Data\_2 & 0.95 & $\mathbf{5.87e-02}$ & 0.27 & 0.59 & $\mathbf{0.12}$ & 0.59 & 84.38\\
Data\_3 & 7.1e-03 & $\mathbf{2.33e-03}$ & 0.19 & 0.55 & $\mathbf{0.05}$ & 0.55 & 16.87\\
Data\_4 & 5.9e-02 & $\mathbf{0.27}$ & 0.33 & 0.37 & $\mathbf{0.14}$ & 0.37 & 71.84\\
\hline
\end{tabular}
\end{center}
\caption{Comparison results on synthetic datasets}
\label{tab1}
\end{table*}

\begin{figure}
        \centering
        \begin{subfigure}[b]{0.5\textwidth}
                \includegraphics[width=\textwidth]{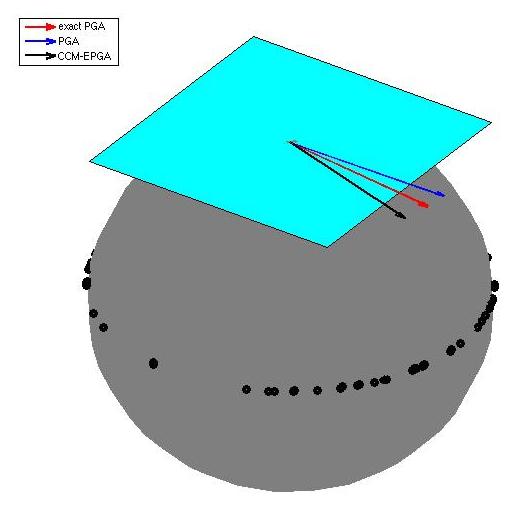}
        \end{subfigure}%
               \caption{Synthetic data (Data\_1) on $S^2$}\label{fig1}
\vspace*{-0.4em}
\end{figure}

\subsection{Comparative performance on point-set data ($\mathbf{S}^N$ example)}

In this section, we depict the performance of the proposed CCM-EPGA on
2D point-set data. The database is called GatorBait-100 dataset
\cite{anand}. This dataset consists of $100$ images of shapes of
different fishes.  From each of these images of size $20\times 200$,
we first extract boundary points, then we apply the Schr{\"o}dinger
distance transform \cite{deng2014riemannian} to map each of these
point sets on a hypersphere ($S^{3999}$). Hence, this data consists of
$100$ point-sets each of which lie on $S^{3999}$. As before, we have
used the {\it average projection error} \cite{sommer2010manifold}, to
measure the performance of algorithms in the
comparisons. Additionally, we report the computation time for each of
these PGA algorithms.  We used the code available online for {\it
  exact PGA} \cite{stefan}. This online implementation is not scalable
to large (even moderate) number of data points, and furhter requires
the computation of the Hessian matrix in the optimization step, which
is computationally expensive. Hence, for this real data application on
the high dimensional hypersphere, we could not report the results for
the {\it exact PGA} algorithm. Though one can use a Sparse matrix
version of the exact PGA code, along with efficient parallelization to
make the {\it exact PGA} algorithm suitable for moderately large data,
we would like to point out that since our algorithm does not need such
modifications, it clearly gives CCM-EPGA an advantage over {\it exact
  PGA} from a computational efficiency perspective. In terms of
accuracy, it can be clearly seen that CCM-EPGA outperforms {\it exact
  PGA} from the results on synthetic datasets. Both {\it average
  projection error} and computational time on GatorBait-100 dataset
are reported in Table \ref{tab2}. This result demonstrates accuracy of
CCM-EPGA over the PGA algorithm with a marginal sacrifice in efficiency
but significant gains in accuracy.

\begin{table}
\begin{center}
\begin{tabular}{|l|c|c|}
\hline
Method & {\it avg. proj. error} & Time(s) \\
\hline\hline
CCM-EPGA & $\mathbf{2.83e-10}$ & $0.40$\\
PGA & $9.68e-02$ & $\mathbf{0.28}$ \\
\hline
\end{tabular}
\end{center}
\caption{Comparison results on \emph{Gator-Bait} database}
\label{tab2}
\end{table}
\subsection{PGA on population of gaussian distributions ($\mathbf{H}^N$ example)}

In this section, we propose a novel scheme to compute principal
geodesic submanifolds for the manifold of gaussian densities. Here, we
use concepts from information geometry presented in
\cite{amari2007methods}, specifically, the Fisher information matrix
\cite{lehmann1998theory} to define a metric on this manifold
\cite{costa2014fisher}. Consider a normal density
$f(.|\boldsymbol\theta)$ in an $n-$dimensional space, with parameters
represented by $\boldsymbol \theta$. Then the ${ij}^{th}$ entry of the
$n\times n$ Fisher matrix, denoted by $g_{ij}$, is defined as follows:
\begin{equation}
g_{ij}(\boldsymbol\theta) = \int_{\mathbf{R}} f(x| \boldsymbol\theta) \frac{\partial ln f(x| \boldsymbol\theta)}{\partial \theta_i} \frac{\partial ln f(x| \boldsymbol\theta)}{\partial \theta_j} dx
\end{equation}
For example, for a univariate normal density $f(.|\mu, \sigma)$, the
fisher information matrix is
\begin{equation}
(g_{ij}(\mu,\sigma)) =\begin{pmatrix}
  \frac{1}{\sigma^2} & 0\\
  0 & \frac{2}{\sigma^2}
 \end{pmatrix}
\end{equation}
So, the metric is defined as follows:
\begin{equation}
<u, v> = u^tGv
\end{equation}
where $G=(g_{ij})$ is the Fisher information matrix. Now, consider the
parameter space for the univariate normal distributions. The parameter
space is $H_F = {(\mu, \sigma) \in \mathbf{R}^2| \sigma > 0}$, i.e.,
positive half space, which is the Hyperbolic space, modeled by the
Poincar{\'e} half-plane, denoted by $\mathbf{P}^2$. We can define a
bijection $F_1: H_F \rightarrow \mathbf{P}^2$ as $F(\mu, \sigma) =
(\frac{\mu}{\sqrt{2}}, \sigma)$. Hence, the univariate normal
distributions can be parametrized by the $2-$dimensional hyperbolic
space. Moreover, there exists a diffeomorphsim between $\mathbf{P}^2$
and $\mathbf{H}^2$ (the mapping is analogous to stereographic
projection for $\mathbf{S}^N$), thus, we can readily use the
formulation in Section \ref{sec2} to compute principal geodesic
submanifold on the manifold of univariate normal distributions.

Motivated by the above formulation, we ask the following question: {\it
  Does there exist a similar relation for multivariate normal
  distributions?} The answer is no in general. But if the multivariate
distributions have diagonal covariance matrix, (i.e., independent
uncorrelated variables in the mutivariate case), the above relation
between $\mathbf{P}^2$ and $\mathbf{H}^2$ can be generalized. Consider
an $N-$dimensional normal distribution parametrized by
$(\boldsymbol\mu, \Sigma)$ where $\boldsymbol\mu = (\mu_1, \cdots,
\mu_N)^t$ and $\Sigma$ is a diagonal positive definite matrix (i.e.,
$\Sigma_{ij} = \sigma_i$, if $i=j$, else $\Sigma_{ij} = 0$). Then,
analogous to the univariate normal distribution case, we can define a
bijection $F_N: H^N_F \rightarrow \mathbf{P}^{2N}$ as follows:
\begin{equation}
F_N(\boldsymbol\mu, \Sigma) = (\frac{\mu_1}{\sqrt{2}}, \sigma_1, \cdots, \frac{\mu_N}{\sqrt{2}}, \sigma_N)
\end{equation}
Hence, we can use our formulation in Section \ref{sec2} since there is
a diffeomorphism between $\mathbf{P}^{2N}$ and $\mathbf{H}^{2N}$. But,
for general non-diagonal $N-$dimensional covariance matrix space,
$SPD(N)$, the above formulation does not hold. This motivated us to go
one step further to search for a parametrization of $SPD(N)$ where we
can use the above formulation. In \cite{jian2007metric}, authors have
used the Iwasawa coordinates to parametrize $SPD(N)$. Using the
Iwasawa coordinates \cite{terras1985harmonic}, we can get a one-to-one
mapping between $SPD(N)$ and the product manifold of $PD(N)$ and
$U(N-1)$, where $PD(N)$ is manifold of $N-$dimensional diagonal
positive definite matrix and $U(N-1)$ is the space of
$(N-1)-$dimensional upper triangular matrices, which is isomorphic to
$R^{N(N+1)/2}$. We have used the formulation in \cite{terras1985harmonic}, 
as discussed below.

Let $Y = V_N \in SPD(N)$, then we can use Iwasawa decomposition to
represent $V_N$ as a tuple $(V_{N-1}, x_{N-1}, w_{N-1})$. And
repeating the following partial Iwasawa decomposition:
\begin{equation}
V_N = \begin{pmatrix}I & x_{N-1} \\0 & 1 \end{pmatrix}^T \begin{pmatrix}V_{N-1} & 0\\ 0 & w_{N-1}\end{pmatrix}
\end{equation}
where $w_{N-1} > 0$ and $x_{N-1} \in \mathbf{R}^{N-1}$. We get the
following vectorized expression: $V_{N} \mapsto (((w_0, x_1^t, w_1),
x_2^t, w_2), \cdots, x_{N-1}^t, w_{N-1})$. Note that as each of $w_i$
is $>0$, we can construct a positive definite diagonal matrix with
$i^{th}$ diagonal entry being $w_i$. And as each $x_i$ is in
$\mathbf{R}^i$, we will arrange them columnwise to form a upper
triangular matrix.  Thus, for $PD(N)$, we can use our formulation for
the hyperboloid model of the hyperbolic space given in Section
\ref{sec2}, and the standard PCA can be applied for $R^{N(N+1)/2}$.

We now use the above formulation to compute the principal geodesic
submanifolds for a covariance descriptor representation of Brodatz
texture dataset \cite{brodatz1966textures}. Similar to our previous
experiment on point-set data, in this experiment, we report the {\it
  average projection error} and computation time. We adopt a similar
procedure as in \cite{ho2013nonlinear} to derive the covariance
descriptors for the texture images in the Brodatz database. Each
$256\times 256$ texture image is first partitioned into $64$
non-overlapping $8\times 8$ blocks. Then, for each block, the
covariance matrix ($FF^T$) is summed over the blocks. Here, the
covariance matrix is computed from the feature vector $F = \bigg(I,
|\frac{\partial I}{\partial x}|, |\frac{\partial I}{\partial y}|,
|\frac{\partial^2 I}{\partial x^2}|, |\frac{\partial^2 I}{\partial
  y^2}|\bigg)^t$. We make the covariance matrix positive definite by
adding a small positive diagonal matrix. Then, each image is
represented as a normal distribution with zero mean and this computed
covariance matrix. Then, we used the above formulation to map each
normal distribution on to $\mathbf{H}^{10}$. The comparative results
of CCM-EPGA with PGA and {\it exact PGA} are presented in Table
\ref{tab3}. The results clearly indicate efficiency and accuracy of
CCM-EPGA compared to PGA and exact PGA algorithms.

\begin{table}
\begin{center}
\begin{tabular}{|l|c|c|}
\hline
Method & {\it avg. proj. error} & Time(s) \\
\hline\hline
CCM-EPGA & $\mathbf{7.73e-03}$ & $0.09$\\
PGA & $0.14$ & $\mathbf{0.05}$ \\
{\it exact PGA} & $0.09$ & $732$ \\
\hline
\end{tabular}
\end{center}
\caption{Comparison results on \emph{Brodatz} database}
\label{tab3}
\end{table}

\section{Data Reconstruction from principal directions and coefficients}
\label{sec4}
In this section, we discuss a recursive scheme to approximate an
original data point with principal directions and coefficients. We
discuss the reconstruction for data on $\mathbf{S}^N$, the
reconstruction for data points on $\mathbf{H}^N$ can be done in an
analogous manner. Let $x_j \in \mathbf{S}^N$ be the $j^{th}$ data
point and $\mathbf{v}_k$ the $k^{th}$ principal vector. $\bar{x}_j^k$
is the $k^{th}$ principal component of $x_j$. Note that on
$\mathbf{S}^N$, $k^{th}$ principal component of a data point is
$y(\mathbf{v}_k, x_j)$. Let $x_j^k$ be the approximated $x_j$ from the first $k$ principal components. Let $\mathbf{w}_{j}^{k-1}$ be
$Log_{\mu}{x_j^{k-1}}/\|Log_{\mu}{x_j^{k-1}}\|$. Let
$\bar{\mathbf{w}}_{j}^{k-1}$ be the parallel transported vector
$\mathbf{w}_{j}^{k-1}$ from $\mu$ to $\bar{x}_j^k$. Let,
$\bar{\mathbf{v}_k}$ be the parallel transported version of $\mathbf{v}_k$
to $x_j^{k-1}$. We refer readers to Fig. \ref{fig311} for a geometric
interpretation.

\begin{figure}
        \centering
        \begin{subfigure}[b]{0.4\textwidth}
                \includegraphics[width=\textwidth]{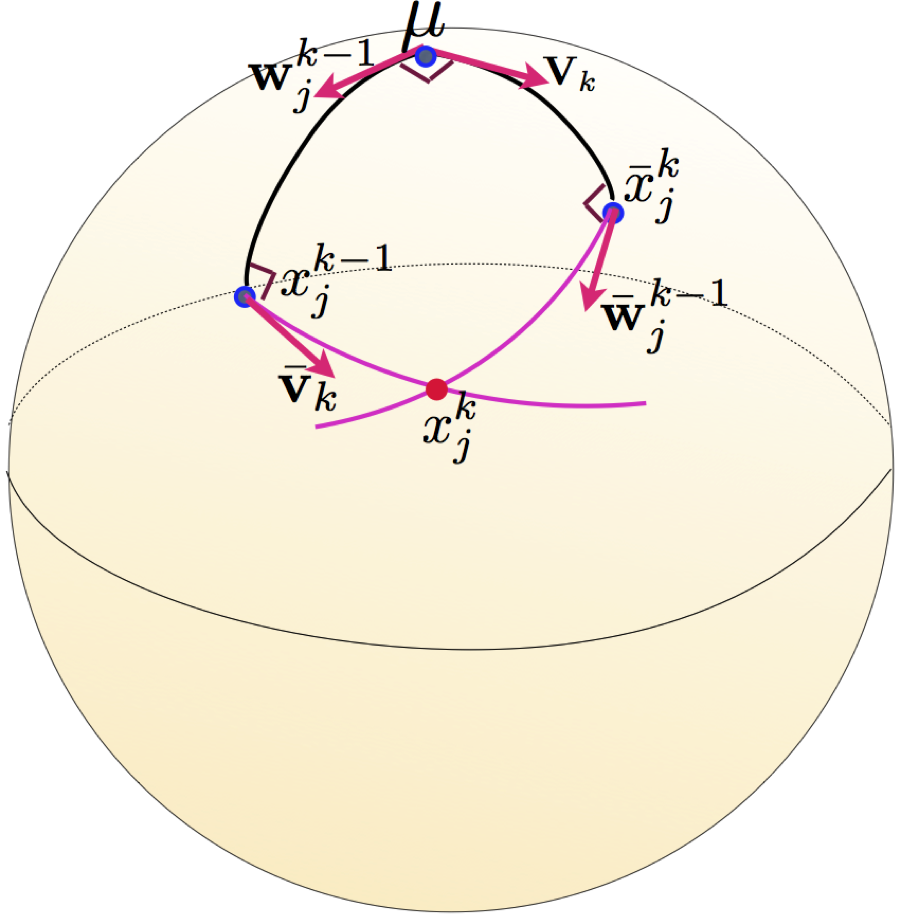}
        \end{subfigure}
\caption{Approximation of $x_j$ from the first $k$ principal components.}\label{fig311}
\end{figure}

Now, we will formulate a recursive scheme to reconstruct $x_j$. Let us
reconstruct the data using the first $(k-1)$ principal
components. Then, the $k^{th}$ approximated point $x_j^k$ is the
intersection of two geodesics defined by $x_j^{k-1}$,
$\bar{\mathbf{v}_k}$ and $\bar{x}_j^k$,
$\bar{\mathbf{w}}_{j}^{k-1}$. Let these two great circles be denoted
by
\begin{eqnarray}
G_1(t) = \cos(t)x_j^{k-1} + \sin(t)\bar{\mathbf{v}_k} \\
G_2(u) = \cos(u)\bar{x}_j^k + \sin(u)\bar{\mathbf{w}}_{j}^{k-1} 
\end{eqnarray}

At $t=\alpha_1$ and $u=\alpha_2$, let $G_1(\alpha_1) = G_2(\alpha_2) =
x_j^k$. Since, $\bar{\mathbf{v}_k}$ and $\bar{\mathbf{w}}_{j}^{k-1}$ are
mutually orthogonal, we get,
\begin{equation}
\tan(\alpha_1)\tan(\alpha_2) = <x_j^{k-1}, \bar{\mathbf{w}}_{j}^{k-1}><\bar{x}_j^k, \bar{\mathbf{v}_k}>
\label{eq211}
\end{equation}
Note that, as our goal is to solve for $\alpha_1$ or $\alpha_2$ to get
$x_j^k$, we need two equations. The second equation can be derived as
follows:
$$
d(\mu, G_1(\alpha_1)) = d(\mu, G_2(\alpha_2))
$$
This leads to,
\begin{equation}
\cos(\alpha_2) = \frac{\cos(\alpha_1)<\mu, x_j^{k-1}>}{<\mu, \bar{x}_j^k>}
\label{eq212}
\end{equation}
Then, by solving Eqs. \eqref{eq211} and \eqref{eq212} we get,
\begin{equation}
a \cos^4(\alpha_1) + b\cos^2(\alpha_1) + d =0
\label{eq213}
\end{equation}
where,
\begin{align*}
a = <\mu, x_j^{k-1}>^2<x_j^{k-1}, \bar{\mathbf{w}}_{j}^{k-1}>^2<\bar{x}_j^k, \bar{\mathbf{v}_k}>^2 \\
- <\mu, x_j^{k-1}>^2
\end{align*}
$$
b = <\mu, \bar{x}_j^k>^2 + <\mu, x_j^{k-1}>^2
$$
and
$$
d = -<\mu, \bar{x}_j^k>^2
$$
By solving the equation \eqref{eq213}, we get
\begin{equation}
\alpha_1 = \arccos \bigg(\sqrt{\frac{-b+\sqrt{(b^2-4ad)}sgn(a)}{2a}}\bigg)
\end{equation}
where, $sgn(.)$ is the signum function. Hence, $x_j^k =
G_1(\alpha_1)$.  This completes the reconstruction algorithm. Our
future efforts will be focused on using this reconstruction algorithm
in a variety of applications mentioned earlier.

\section{Conclusions}
\label{sec5}
In this paper, we presented an efficient and accurate exact-PGA
algorithm for (non-zero) constant curvature manifolds, namely the
hypersphere $S^n$ and the hyperbolic space $H^n$. We presented
analytic expressions for the shortest distance between the data point
and the geodesic submanifolds, which is required in the PGA algorithm
and in general involves solving a difficult optimization
problem. Using these analytic expressions, we achieve a much more
accurate and efficient solution for PGA on constant curvature
manifolds, that are frequently encountered in Computer Vision, Medical
Imaging and Machine Learning tasks. We presented comparison results
on synthetic and real data sets demonstrating favorable performance of
our algorithm in comparison to the state-of-the-art.

{\small
\bibliographystyle{ieee}
\bibliography{egbib}
}
\end{document}